%% file: main.tex
\documentclass[10pt,letterpaper]{article}
\usepackage[preprint]{icml2026}

\input{preamble}

\usepackage{graphicx}
\usepackage{booktabs}
\usepackage{multirow}
\usepackage{xcolor}
\usepackage{pifont}
\usepackage[T1]{fontenc}
\usepackage{dsfont}
\usepackage{subcaption}
\usepackage{amsmath, amssymb, bm, amsthm}
\newtheorem{proposition}{Proposition}
\newtheorem{theorem}{Theorem}
\newtheorem{lemma}{Lemma}
\newtheorem{assumption}{Assumption}

\usepackage[colorlinks=true, citecolor=blue, urlcolor=blue, linkcolor=blue]{hyperref}

\icmltitlerunning{FastCache: Fast Caching for Diffusion Transformer}

\begin{document}

\twocolumn[
  \icmltitle{FastCache: Fast Caching for Diffusion Transformer Through Learnable Linear Approximation}


  \begin{icmlauthorlist}
    \icmlauthor{Dong Liu$^{*}$}{aff1,aff3,aff4}
    \icmlauthor{Yanxuan Yu}{aff2}
    \icmlauthor{Jiayi Zhang}{aff4}
    \icmlauthor{Yifan Li}{aff5}
    \icmlauthor{Ben Lengerich}{aff4}
    \icmlauthor{Ying Nian Wu}{aff3}
  \end{icmlauthorlist}

  \icmlaffiliation{aff1}{Yale University}
  \icmlaffiliation{aff2}{Columbia University}
  \icmlaffiliation{aff3}{University of California, Los Angeles}
  \icmlaffiliation{aff4}{University of Wisconsin--Madison}
  \icmlaffiliation{aff5}{Michigan State University}

  \icmlcorrespondingauthor{Dong Liu}{pikeliu@ucla.edu}

  \icmlkeywords{Caching, Diffusion Transformers, Inference Acceleration, Token Reduction}

  \vskip 0.3in
]

\printAffiliationsAndNotice{}

\input{contents/0_abst}

\section{Introduction}
\input{contents/1_intro}

\section{Related Work}
\input{contents/2_rw}

\input{contents/3_method_brief}

\input{contents/4_interp}

\section{Experiments}
\input{contents/5_expr}

\section{Conclusion}
\input{contents/6_clu}

{
    \small
    \bibliographystyle{icml2026}
    \bibliography{main}
}






\newpage
\onecolumn
\appendix
\input{contents/8_appendix_proof}
\input{contents/9_appendix_expr}

\end{document}

%% file: preamble.tex
\usepackage[dvipsnames]{xcolor}


%% file: contents/0_abst.tex
\begin{abstract}
Diffusion Transformers (DiT) are powerful generative models but remain computationally intensive due to their iterative structure and deep transformer stacks. 
To alleviate this inefficiency, we propose \textit{\textbf{FastCache}}, a hidden-state-level caching and compression framework that accelerates DiT inference by exploiting redundancy within the model’s internal representations. 
\textit{FastCache} introduces a dual strategy: (1) a spatial-aware token selection mechanism that adaptively filters redundant tokens based on hidden state saliency, and (2) a transformer-level cache that reuses latent activations across timesteps when changes are statistically insignificant. 
These modules work jointly to reduce unnecessary computation while preserving generation fidelity through learnable linear approximation.
Theoretical analysis shows that \textit{FastCache} maintains bounded approximation error under a hypothesis-testing-based decision rule. 
Empirical evaluations across multiple DiT variants demonstrate substantial reductions in latency and memory usage, with best generation output quality compared to other cache methods, as measured by FID and t-FID. To further improve the speedup performance of FastCache, we also introduce a token merging module that merge redundant tokens based on kNN density. 
Code is available at \href{https://github.com/NoakLiu/FastCache-xDiT}{https://github.com/NoakLiu/FastCache-xDiT}.
\end{abstract}

%% file: contents/1_intro.tex
Transformer-based diffusion models have emerged as powerful generative frameworks for image and video synthesis. Among them, Diffusion Transformers (DiT)~\cite{dit} leverage hierarchical attention mechanisms and deep transformer stacks to achieve high-fidelity generation results.
However, the reliance on multi-level attention structures and extensive transformer depth introduces considerable computational overhead. In particular, repeated transformer operations across timesteps and spatial locations substantially increase inference cost, especially in scenarios where hidden states exhibit minimal variation. This significantly constrains the practicality of DiT in real-time applications such as video editing, virtual reality, and large-scale content generation.

Despite progress in accelerating diffusion inference via architectural modifications or token pruning, such as ~\cite{tomesd, tokenfusion, pab, adacache}, current methods still struggle to balance efficiency and generation quality.
Most notably, existing approaches focus on the frame or token level , treating individual frames or spatial patches as the unit of compression.
Yet, diffusion generation is inherently redundant at the \emph{representation level}, as the hidden states across successive timesteps often exhibit minimal changes, especially in low-motion regions or later denoising steps.
Ignoring such redundancy leads to unnecessary recomputation and inflated resource usage.


\textbf{How can we exploit redundancy within hidden states to accelerate inference?}
While prior works such as AdaCache~\cite{adacache} and DeepCache~\cite{deepcache} attempt to reuse features across timesteps, they focus on low-level U-Net structures or specific attention layers.
In contrast, we explore caching from a \textit{transformer-level} perspective in DiT, targeting hidden states as the fundamental unit of reuse.
This perspective enables finer-grained and model-agnostic acceleration strategies, which can scale across different DiT configurations.

To this end, we propose \textbf{\emph{FastCache}}, a spatial-temporal caching framework that operates directly on transformer hidden states to accelerate DiT inference. 
\textit{FastCache} comprises two core components:
(1) A hidden-state caching mechanism that dynamically decides whether to reuse or recompute each timestep's representation based on statistical similarity tests; 
(2) A saliency-aware token reduction module that selectively prunes redundant spatial tokens based on motion and content variations.
Together, these modules adaptively compress computation in both the temporal and spatial dimensions without degrading generation quality.

The decision to reuse cached states is guided by a rigorous statistical hypothesis test on hidden-state changes, allowing \textit{\textbf{FastCache}} to balance computational savings and error tolerance. This design is further supported by theoretical analysis showing that the cache error is tightly bounded under mild distributional assumptions.
\begin{figure}
    \centering
    \includegraphics[width=1\linewidth]{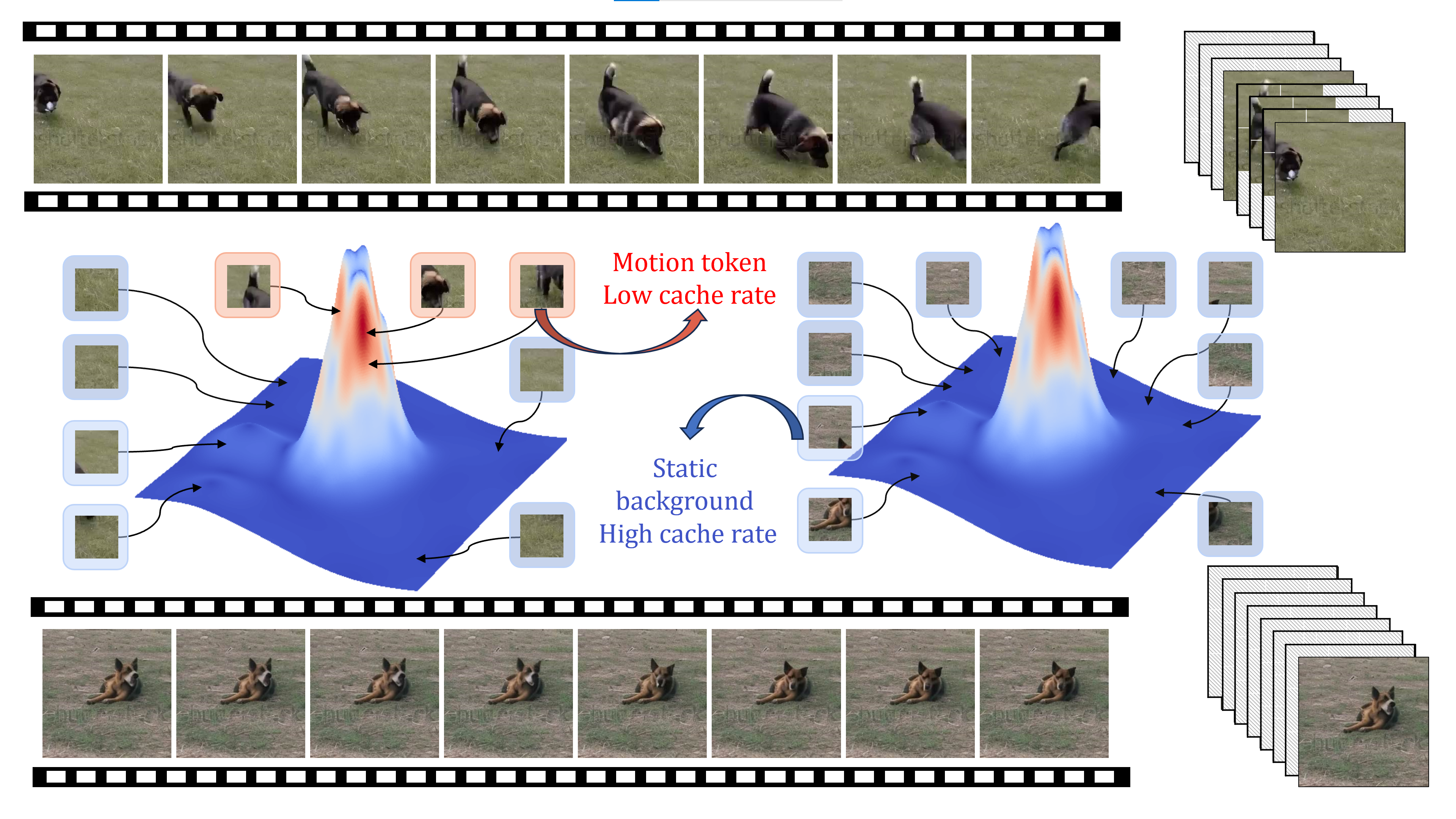}
    \caption{
    \textbf{Top:} A video exhibiting substantial motion, with the corresponding masked input displayed to the right. The masked regions indicate areas excluded from DiT processing. 
    \textbf{Middle:} An illustration of the hidden-state interaction space, where the colormap visualizes the magnitude of first-order derivatives. Warmer colors (e.g., red) represent higher absolute gradient values, while cooler colors (e.g., blue) indicate lower values. 
    \textbf{Bottom:} A comparatively static video, with its masked input shown on the right. 
    \textbf{Interpretation:} Videos characterized by greater frame-to-frame variation, such as the top example, contain a higher proportion of motion-related tokens. These tokens are associated with the red regions in the heatmap and tend to exhibit a lower cache utilization rate (\textbf{motion region needs recomputation}), these tokens are associated with the blue regions in the heatmap and tend to exhibit a higher cache utilization rate (\textbf{static background will be cached by \textit{FastCache} for acceleration})
    }
    \label{fig:overview}
\end{figure}
\vspace{0.5em}
\noindent \textbf{Our contributions are summarized as follows:}
\begin{itemize}
    \item We introduce \textit{FastCache}, a \textbf{hidden-state-level caching and compression framework} for accelerating DiT inference via spatial-temporal redundancy exploitation.
    \item We develop a \textbf{motion-aware token pruning strategy} that identifies and suppresses redundant spatial tokens across timesteps.
    \item We provide a \textbf{theoretical justification for the interpretability} of \textit{FastCache} by framing it as a learnable interaction decomposition framework, providing a rigorous foundation for symbolic motion-background separation via sparse linear approximation of hidden dynamics.
\end{itemize}

%% file: contents/2_rw.tex
\vspace{0.5em}

\noindent\textbf{Efficiency in Generative Models.}
Generative models—including Transformers and Diffusion architectures—achieve state-of-the-art performance across language, vision, and multimodal tasks, but remain bottlenecked by their compute-intensive inference pipelines. ~\cite{liu2024contemporary,choudhary2020comprehensive,ganesh2021compressing}
To address this, recent efforts span multiple levels of optimization, from quantization and pruning~\cite{liu2025llmeasyquant,liu2024contemporary,lin2024awq,liu2018rethinking,ma2023llm} to token-level sparsity~\cite{bolya2023tomevit,wang2021spatten,rao2021dynamicvit}, and system-level scheduling~\cite{xiao2023strllm,zhang2023h2o,vllm2023}.
Meanwhile, adaptive inference strategies such as early exit~\cite{teerapittayanon2016branchynet,elbayad2020depthadaptive}, dynamic scheduling~\cite{liumt2st,neseem2023adamtl}, and query-aware routing~\cite{schuster2022confidentadaptive,ding2024hybrid} adjust computation based on estimated token importance to reduce redundant processing.
Despite these advances, most approaches focus on early representations or static configurations, leaving hidden-state reuse and cross-layer adaptation largely underexplored.

\noindent\textbf{Spatial-Temporal Optimization in Diffusion Models.}
Diffusion models have achieved impressive results in high-quality image and video generation. However, their multi-step iterative inference is computationally expensive, especially for video or high-resolution scenarios. 
To alleviate this, recent works explore spatial-temporal compression techniques that reduce redundant computation across both token space and temporal steps.
Diffusion-SpaceTime-Attn~\cite{diffusion_spacetime_attn} employs spatiotemporal attention masks for selective computation. 
Latte~\cite{latte} adopts lightweight spatial embeddings, while TeaCache~\cite{teacache} introduces keyframe caching to avoid redundant decoding.
ParaAttention~\cite{ParaAttention} achieves spatial-temporal parallelism by caching the first block across timesteps.
Beyond temporal skipping, several methods exploit structural sparsity: BlockDance~\cite{blockdance}, DualCache~\cite{dual_feature_caching}, and Error-Optimized Cache~\cite{error_optimized_cache} reuse previously computed activations to minimize redundant work. 
Other efforts such as FORA~\cite{fora} and FasterCache~\cite{fastercache} adopt layer-wise cache policies to improve fine-grained reuse granularity, while token-level approaches like TokenCaching~\cite{token_caching} adaptively select key tokens for reuse.
Although effective, these approaches often operate at the frame level or early visual features, ignoring the rich redundancy within deeper latent representations.

\vspace{0.8em}
\noindent\textbf{Hidden-State Caching in Diffusion and Transformers.}
Caching mechanisms have long been used in neural networks to reduce inference costs, such as in BERT acceleration or transformer decoding. 
In diffusion, DeepCache~\cite{deepcache} reuses skip-connection features in UNet when frame difference is low, while TGATE~\cite{tgate} bypasses partial attention layers using a time-gated mechanism.
However, these techniques are designed for UNet architectures and are not applicable to DiT, which relies on stacked attention blocks.
Recent methods aim to accelerate DiT directly. 
PAB~\cite{pab} introduces a pyramidal attention broadcasting strategy, where attention is reused across layers with fixed frequency.
AdaCache~\cite{kahatapitiya2024adaptive} further introduces input-similarity-based reuse, allowing transformer blocks to skip recomputation over time. 
While effective, these methods either lack theoretical guarantees or rely on additional priors like learned codebooks or fixed frequency patterns, which may limit generalization across datasets or architectures.

\vspace{0.8em}
\noindent\textbf{Token Pruning, Mixture-of-Depth, and Distillation.}
Orthogonal to caching, various strategies have been proposed to reduce inference cost via model simplification. 
Mixture-of-Depth (MoD) methods such as AdaNet~\cite{adanet}, SkipNet~\cite{skipnet}, and DepthControlNet~\cite{depthcontrolnet} dynamically adapt network depth per input.
Knowledge distillation methods like FastVideo~\cite{fastvideo} and DINO~\cite{dino} train smaller student models under the guidance of large diffusion models.
Meanwhile, token merging and pruning methods (e.g., ToMe~\cite{tome}, ToMeSD~\cite{tomesd}, TokenFusion~\cite{tokenfusion}) compress spatial tokens to minimize computation without degrading quality.
Although these methods improve efficiency, they do not exploit temporal redundancy or the reuse potential of latent transformer states.
In contrast, our proposed \textbf{FastCache} directly targets hidden-state caching within DiT blocks, offering a theoretically grounded framework for selective reuse under bounded approximation error.
By operating at the representation level, FastCache complements existing spatial and temporal compression strategies, and offers scalability across different model sizes and diffusion schedules.

\vspace{0.8em}
\noindent\textbf{Diffusion Model Interpretability.}
Diffusion models are capable of generating high-quality images through iterative denoising processes. However, their internal mechanisms remain largely opaque, presenting significant challenges for interpretability. Recent research has sought to illuminate these processes through various visualization and explanation techniques. For example, DF-RISE~\cite{park2024explaining} provides insights into the denoising dynamics by identifying input regions that most influence the model’s predictions across timesteps. In the context of Vision Transformers, the IA-RED² ~\cite{iared} framework introduces a redundancy-aware strategy that selectively skips uninformative spatial patches, thereby reducing computational overhead. Nevertheless, this method addresses only spatial redundancy at the patch level and overlooks temporal redundancy, namely the repeated computations across diffusion timesteps where latent representations often change minimally. We offer a theoretical justification for the interpretability of our approach. In particular, our spatiotemporal caching metric lends itself to intuitive analysis via interaction heatmaps, providing transparent insights into the model’s internal behavior.

\vspace{0.8em}
\noindent\textbf{Zero-Shot Redundancy Reduction.}
Recent approaches to accelerating inference in diffusion models often reuse outputs from previous steps to bypass redundant computations, as exemplified by methods such as DeepCache~\cite{deepcache,teacache,adacache}. However, these strategies frequently fail to capture subtle variations in token representations, which leads to generated videos exhibiting static or frozen backgrounds, particularly in low-motion scenarios. This limitation often results in elevated FID scores, reflecting a degradation in perceptual quality. More recent methods, such as LazyDiT~\cite{lazydit}, attempt to mitigate this by employing linear approximations that blend cached content with prior outputs. Nevertheless, simple linear approximation remains inadequate for capturing nuanced changes in latent representations. To address this limitation, we aim to strike a balance between training cost and inference efficiency. We propose a lightweight linear layer to substitute the skipped transformer blocks for inactive tokens, effectively reducing the FID and enhancing perceptual quality.

%% file: contents/3_method_brief.tex
\section{Method}
\label{sec:method}

\subsection{Overview}
\label{sec:overview}
The goal of FastCache is to accelerate the inference of Diffusion Transformers (DiT) by eliminating redundant computations across spatial and temporal dimensions.
It introduces two core modules: the \textbf{Spatial-Temporal Token Reduction Module} and the \textbf{Transformer-Level Caching Module}, as illustrated in Figure~\ref{fig:architecture}.


Given hidden states $X_t \in \mathbb{R}^{N \times D}$ at timestep $t$, FastCache dynamically identifies redundant spatial tokens and non-salient transformer blocks for approximation.

\subsection{Spatial-Temporal Token Reduction Module}
\label{sec:spatial_module}

Given hidden states $X_t$ and a previous state $X_{t-1}$, define temporal saliency:
\begin{equation}
    S_t = \|X_t - X_{t-1}\|_2^2.
\end{equation}
We threshold each token to separate static and motion tokens:
\begin{equation}
    \mathcal{M}_t = \{i : S_t^{(i)} > \tau_s\}, \quad X_t^m = X_t[\mathcal{M}_t], \quad X_t^s = X_t[\bar{\mathcal{M}}_t].
\end{equation}

Motion tokens $X_t^m$ are retained, while static tokens $X_t^s$ are bypassed across all transformer layers using a computation-efficient linear approximation:
\begin{equation}
    H_t^s = W_c X_t^s + b_c.
\end{equation}

Both token sets are then concatenated for subsequent stages.

\subsection{Transformer-Level Caching Module}
\label{sec:caching_module}

For each transformer block $l$ in a sequential stack, we decide whether to reuse cached output $H_{t,l-1}$ or execute computation.

\textbf{Relative change metric:} Let $H_{t,l-1}$ be the hidden state before block $l$, and $H_{t-1,l-1}$ the corresponding cached state from the previous timestep. Define:
\begin{equation}
    \delta_{t,l} = \frac{\|H_{t,l-1} - H_{t-1,l-1}\|_F}{\|H_{t-1,l-1}\|_F}.
\end{equation}

Under weak stationarity assumptions $H_{t,l-1} \sim \mathcal{N}(\mu, \Sigma)$, we approximate:
\begin{equation}
    (ND) \cdot \delta_{t,l}^2 \sim \chi^2_{ND}.
\end{equation}

\textbf{Cache decision rule:} For confidence level $1 - \alpha$, skip transformer block $l$ and approximate its output as:
\begin{equation}
    H_{t,l} = W_l H_{t,l-1} + b_l
\end{equation}
if
\begin{equation}
    \delta_{t,l}^2 \leq \frac{\chi^2_{ND, 1 - \alpha}}{ND},
\end{equation}
otherwise compute:
\begin{equation}
    H_{t,l} = \text{Block}_l(H_{t,l-1}).
\end{equation}

\textbf{Error bound:} For type-II cache usage, the deviation is bounded by:
\begin{equation}
    \epsilon_{\text{cache}} \leq \sqrt{\frac{\chi^2_{ND, 1 - \alpha}}{ND}}.
\end{equation}

\begin{algorithm}[ht]
\caption{FastCache: Spatial-Temporal Caching in Sequential Transformer Blocks}
\label{alg:fastcache}
\begin{algorithmic}[1]
\REQUIRE Hidden state $H_t$, previous hidden $H_{t-1}$, Transformer blocks $\{\text{Block}_l\}_{l=1}^L$, thresholds $\tau_s$, $\alpha$
\ENSURE Output $H_t^L$
\STATE Compute token-wise saliency $S_t \leftarrow \|H_t - H_{t-1}\|_2^2$
\STATE Partition tokens into $X_t^m$, $X_t^s$ based on $\tau_s$
\STATE Initialize $H_{t,0} \leftarrow \text{Concat}(X_t^m, X_t^s)$
\FOR{$l = 1$ to $L$}
    \STATE $\delta_{t,l} \leftarrow \|H_{t,l-1} - H_{t-1,l-1}\|_F / \|H_{t-1,l-1}\|_F$
    \IF{$\delta_{t,l}^2 \leq \chi^2_{ND, 1 - \alpha}/ND$}
        \STATE $H_{t,l} \leftarrow W_l H_{t,l-1} + b_l$ \COMMENT{Linear Approximation}
    \ELSE
        \STATE $H_{t,l} \leftarrow \text{Block}_l(H_{t,l-1})$ \COMMENT{Full computation}
    \ENDIF
\ENDFOR
\STATE \textbf{return} $H_t^L$
\end{algorithmic}
\end{algorithm}

\vspace{0.5em}
\noindent
A visualization of FastCache across timesteps is provided in Figure~\ref{fig:architecture}, demonstrating adaptive computation skipping.

\subsection{Spatial-Temporal Token Merging}
To further minimize redundancy while preserving dynamic details, we propose a multi-criteria token importance metric that considers both local spatial density and temporal saliency. The detailed execution flow of this process is provided in Algorithm \ref{alg:fastcache_plus}.
\begin{equation}
\rho_{sp, i} = \exp\left(-\frac{1}{K} \sum_{j \in \text{kNN}(i)} \|h_{t,i} - h_{t,j}\|_2^2\right),
\end{equation}
where $\text{kNN}(i)$ denotes the set of $K$ nearest neighbors for token $i$. A higher $\rho_{sp, i}$ indicates that the token is centrally located within a dense feature cluster.
Second, to capture motion dynamics, we compute the token-wise temporal saliency $\rho_{tm, i}$ using the $L_2$ norm of the hidden state difference:
\begin{equation}
\rho_{tm, i} = \|h_{t,i} - h_{t-1,i}\|_2.
\end{equation}

We then combine these spatial and temporal metrics into a unified importance score $\mathcal{S}_i$:
\begin{equation}
\mathcal{S}_i = \rho_{sp, i} \cdot (1 + \lambda \cdot \rho_{tm, i}),
\end{equation}
where $\lambda$ is a hyperparameter that modulates the influence of temporal changes. This score ensures that tokens are preserved if they are either spatially representative or contain significant temporal updates.
Based on the score $\mathcal{S}$, we employ the Local Clustering-based Token Merge (Local CTM) strategy. Tokens are grouped into clusters, and the tokens within a cluster $C_k$ are merged into a single representative token $\tilde{h}_k$ via weighted averaging:
\begin{equation}
\tilde{h}_k = \frac{\sum_{j \in C_k} \mathcal{S}_j h_{t,j}}{\sum_{j \in C_k} \mathcal{S}_j}.
\end{equation}
Finally, the merged tokens are processed and eventually restored to the original resolution using the stored mapping $\mathcal{M}$.
\begin{figure}
    \centering
    \includegraphics[width=1\linewidth]{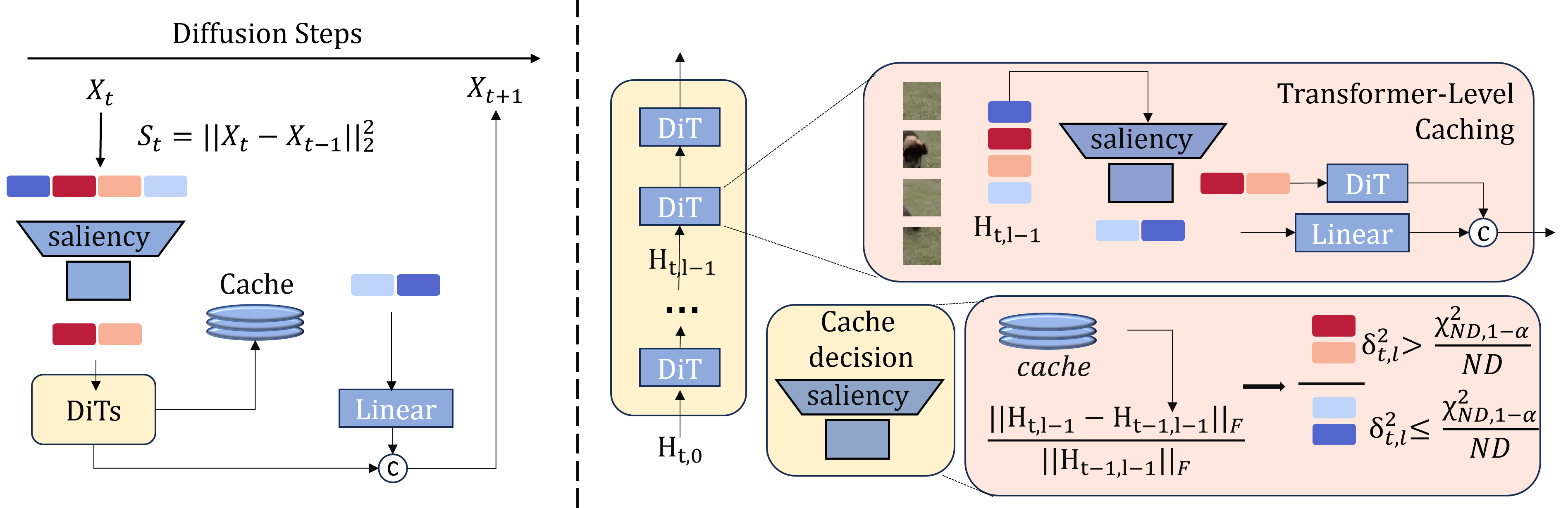}
    \caption{Left: Illustration of the Spatial-Temporal Token Reduction Module. Right: Illustration of the Transformer-Level Caching Module.}
    \label{fig:architecture}
\end{figure}

%% file: contents/4_interp.tex
\section{Interpretability of FastCache: Learnable Linear Approximation for Motion-Background Separation}
\label{sec:theory_fastcache}

In this section, we formulate a theoretical framework to explain FastCache's hidden-state-level compression as a form of learnable sparse interaction decomposition. Specifically, we show that FastCache implements a first-order approximation of transformer hidden dynamics by separating background and motion components via cached projections, supported by interaction theory.

\subsection{Preliminaries: Hidden States and Redundancy}
Let $X^t \in \mathbb{R}^{N \times D}$ denote the hidden states of $N$ spatial tokens at timestep $t$ within a DiT block. FastCache decomposes $X^t$ into a background component $B^t$ and a motion (residual) component $M^t$:
\begin{equation}
X^t = B^t + M^t.
\end{equation}

To construct $B^t$, we use a temporal linear regression over $k$ past hidden states:
\begin{equation}
B^t_i = \theta_0 + \sum_{j=1}^k \theta_j X^{t-j}_i, \quad i=1,\dots,N,
\end{equation}
where $\theta_j \in \mathbb{R}^{D \times D}$ are learned or fit via least squares. The motion signal is then computed as:
\begin{equation}
M^t_i = X^t_i - B^t_i.
\end{equation}

\subsection{Interaction-Theoretic Interpretation}
Let $v(X)$ denote a scalar-valued scoring function (e.g., diffusion likelihood or perceptual loss). Let $\bm{x} = [x_1, \dots, x_N]$ be the token-wise inputs to $v$. We define masked inputs $\bm{x}_S$ where tokens not in $S$ are masked to a baseline $\bm{b}$.

\textbf{Definition (Harsanyi Interaction)}: For $S \subseteq N$, the interaction effect is:
\begin{equation}
I(S) = \sum_{T \subseteq S} (-1)^{|S| - |T|} \left( v(\bm{x}_T) - v(\bm{x}_\emptyset) \right).
\end{equation}

\textbf{Assumption A1 (Low Interaction in Static Regions)}: For low-motion regions, high-order interactions are negligible:
\begin{equation}
\forall S \subseteq N, |S| > 1, \quad |I(S)| < \varepsilon.
\end{equation}

Under this assumption, we approximate $v$ using only singleton interactions:
\begin{equation}
v(\bm{x}) \approx v(\bm{b}) + \sum_{i=1}^N I(\{i\}) = v(\bm{b}) + \sum_{i=1}^N \phi(i),
\end{equation}
where $\phi(i)$ is the Shapley value of token $i$.

\textbf{Cache Trigger Criterion}: FastCache reuses cached background for token $i$ if its marginal change is small:
\begin{equation}
|\phi^t(i) - \phi^{t-1}(i)| < \tau_c \quad \Rightarrow \quad \text{reuse cache for } i.
\end{equation}

\subsection{Theoretical Proposition: First-order Interaction Recovery}

\begin{proposition}
Let $v : \mathbb{R}^{N \times D} \rightarrow \mathbb{R}$ be a thrice-differentiable function over hidden states. Suppose:
\begin{itemize}
  \item $\max_{i,j,k} \left| \frac{\partial^3 v}{\partial x_i \partial x_j \partial x_k} \right| \le C$ (bounded third-order derivatives);
  \item $\|X^t - B^t\|_2 \le \delta$.
\end{itemize}
Then, the first-order Taylor approximation:
\begin{equation}
v(X^t) = v(B^t) + \nabla v(B^t)^\top (X^t - B^t) + \mathcal{O}(\delta^2)
\end{equation}
is equivalent to recovering first-order interactions:
\begin{equation}
v(X^t) \approx v(B^t) + \sum_{i=1}^N I(\{i\}) + \mathcal{O}(\delta^2).
\end{equation}
\end{proposition}

\begin{proof}
Apply multivariate Taylor expansion at $B^t$:
\begin{equation}
\begin{split}
v(X^t) ={} & v(B^t) + \sum_{i=1}^N \frac{\partial v}{\partial x_i}(B^t) (X^t_i - B^t_i) \\
&+ \frac{1}{2} \sum_{i,j} \frac{\partial^2 v}{\partial x_i \partial x_j}(B^t) (X^t_i - B^t_i)(X^t_j - B^t_j) + \cdots
\end{split}
\end{equation}
Since $\|X^t - B^t\| = \mathcal{O}(\delta)$ and $v$ has bounded 2nd and 3rd-order derivatives, higher-order terms are $\mathcal{O}(\delta^2)$.

By Harsanyi decomposition, singleton interaction effects:
\begin{equation}
I(\{i\}) = v(x_i, \bm{b}_{\setminus i}) - v(\bm{b}) = \frac{\partial v}{\partial x_i}(\bm{b})(x_i - b_i) + \mathcal{O}(\delta^2),
\end{equation}
which matches the linear expansion.
\end{proof}
\subsection{Interpretability Insight: Symbolic Decomposition via Caching}
FastCache identifies low-variation tokens $i$ whose first-order interactions $I(\{i\})$ remain stable. These are encoded in $B^t$. The residual motion $M^t$ captures reactivated high-order interactions or nontrivial variations.

\textbf{Conclusion.} FastCache projects hidden representations onto a sparse, symbolic subspace of first-order interactions. This justifies its compression ability and interprets caching as symbolic knowledge reuse.

%% file: contents/5_expr.tex
\subsection{Experimental Setup}

\textbf{Datasets.}
We evaluate FastCache on four DiT variants using the ImageNet-based high-resolution generation setting. Specifically, we test on the following transformer backbones:
DiT-XL/2, DiT-L/2, DiT-B/2, and DiT-S/2, which vary in model depth and embedding dimensions.
All models are trained with Stable Diffusion’s pretrained VAE encoder for the latent representation, using 50K image or video samples for generation during evaluation.
For video generation, we simulate motion conditions by sampling sequences with fixed start frames and variable temporal dynamics, following standard diffusion inference schedules.
We also include additional experiments on long-horizon sampled sequences (e.g., 32 or 64 frames) to stress-test the temporal cache mechanism.

\vspace{0.5em}
\subsection{Implementation Details}

Experiments are implemented on the DiT baseline by Meta \cite{dit}. Experiments are run on 8 NVIDIA A100 GPUs.
We adopt default inference settings from DiT baselines, with 50 denoising steps and classifier-free guidance enabled. For FastCache, we set the motion cache threshold $\tau_m = 0.05$, background update factor $\alpha = 0.7$, and blending factor $\gamma = 0.5$. For statistical caching, we set the significance level $\alpha = 0.05$, and use a sliding window to track $\delta_t$. 

\subsection{Comparison with State-of-the-art Methods}

We compare FastCache against multiple recent DiT acceleration baselines, including: TeaCache~\cite{teacache}, AdaCache~\cite{kahatapitiya2024adaptive}, PAB~\cite{pab}, and L2C~\cite{fastercache}. Table~\ref{tab:main_results} reports the FID, t-FID, runtime, and memory usage on DiT-XL/2.

\begin{table*}[!htbp]
  \centering
  \caption{\textbf{Comparison with acceleration baselines on DiT-XL/2.} FastCache achieves the best balance between quality and efficiency—surpassing AdaCache in FID and t-FID while maintaining competitive speed.}
  \label{tab:main_results}
  \vspace{0.1em}
  \footnotesize
  \setlength\tabcolsep{4pt}
  \begin{tabular}{l|cc|cc}
    \toprule
    \multirow{2}{*}{Method} & FID $\downarrow$ & t-FID $\downarrow$ & Time (ms) $\downarrow$ & Mem (GB) $\downarrow$ \\
    \midrule
    TeaCache~\cite{teacache} & 5.09 & 14.72 & \textbf{14953} & 12.7 \\
    AdaCache~\cite{kahatapitiya2024adaptive} & 4.64 & 13.55 & 21895 & 14.8 \\
    Learning-to-Cache~\cite{learning_to_cache} & 6.88 & 16.02 & 16312 & \textbf{9.4} \\
    FBCache \cite{ParaAttention} & 4.48 & 13.22 & 16871 & 11.5 \\
    \textbf{FastCache (Ours)} & \textbf{\underline{4.46}} & \textbf{\underline{13.15}} & \underline{15875} & \underline{11.2} \\
    \bottomrule
  \end{tabular}
\end{table*}

FastCache achieves a 36.9\% reduction in runtime and 31.7\% reduction in memory, while the best generation quality compared to other cache methods in FID and t-FID. Compared to other baseline caching method, our method provides better efficiency (good tradeoff between reduction of time and memory) with best generation quality.

\subsection{Ablation Study}

\textbf{Effect of each module.}
To validate the effectiveness of each module in FastCache, we conduct an ablation study on DiT-L/2 using combinations of spatial token reduction (STR), statistical cache (SC), and motion-aware blending (MB). Results are summarized in Table~\ref{tab:ablation}.


\begin{table}[!htbp]
  \centering
  \caption{\textbf{Ablation study on DiT-L/2.} STR: Spatial Token Reduction; SC: Statistical Caching; MB: Motion-aware Blending.}
  \label{tab:ablation}
  \vspace{0.2em}
  \setlength\tabcolsep{5pt}
  \begin{tabular}{ccc|c}
    \toprule
    STR & SC & MB & Time (ms) $\downarrow$ \\
    \midrule
    \ding{55} & \ding{55} & \ding{55} & 22041 \\
    \ding{51} & \ding{55} & \ding{51} & 18972 \\
    \ding{55} & \ding{51} & \ding{51} & 19385 \\
    \ding{51} & \ding{51} & \ding{55} & 17518 \\
    \ding{51} & \ding{51} & \ding{51} & \textbf{16593} \\
    \bottomrule
  \end{tabular}
\end{table}


All three modules contribute to speedup, with STR providing the largest gain and SC+MB offering complementary improvements. The best performance is achieved when all are combined.

\textbf{Statistical threshold sensitivity.}
We vary the significance level $\alpha$ for the statistical test and observe caching ratio vs. FID score (Figure~\ref{fig:stat_sensitivity}). Results show stability under $\alpha \in [0.01, 0.1]$, validating the robustness of our statistical reuse rule.

\begin{figure}[!htbp]
    \centering
    \includegraphics[width=0.38\textwidth]{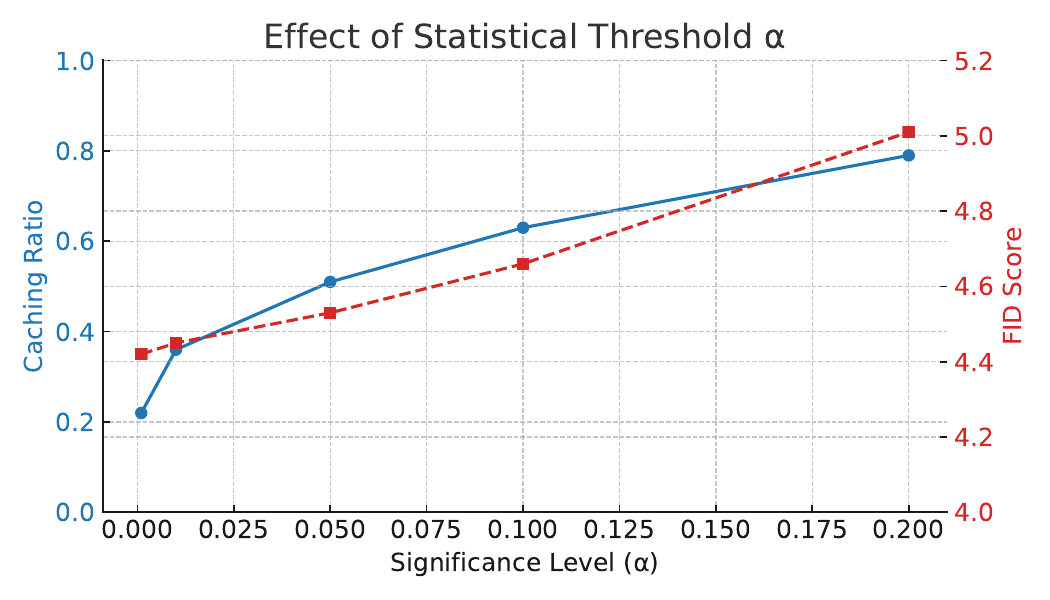}
    \caption{\textbf{Impact of statistical threshold $\alpha$ on caching rate and FID.}}
    \label{fig:stat_sensitivity}
\end{figure}

\subsection{Qualitative Results and Visualization}

Figure~\ref{fig:fastcache_vis_compression} provides qualitative comparisons of image generations with and without FastCache. Despite strong compression, FastCache maintains sharpness and motion coherence across frames.

\begin{figure}[htbp]
  \centering
  \begin{subfigure}[t]{0.45\textwidth}
    \centering
    \includegraphics[width=\textwidth]{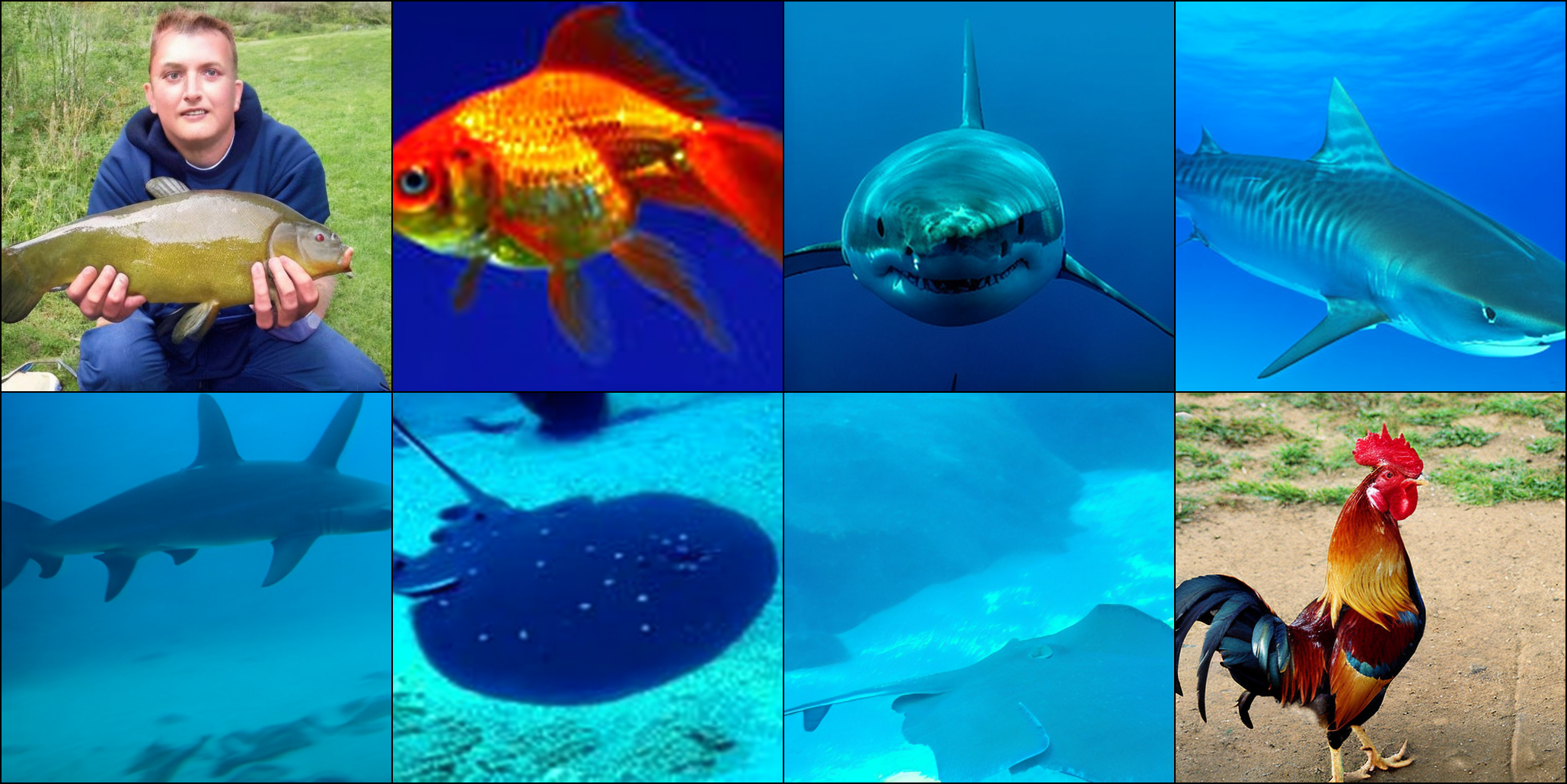}
    \caption{Original DiT Generation}
  \end{subfigure}
  \hfill
  \begin{subfigure}[t]{0.45\textwidth}
    \centering
    \includegraphics[width=\textwidth]{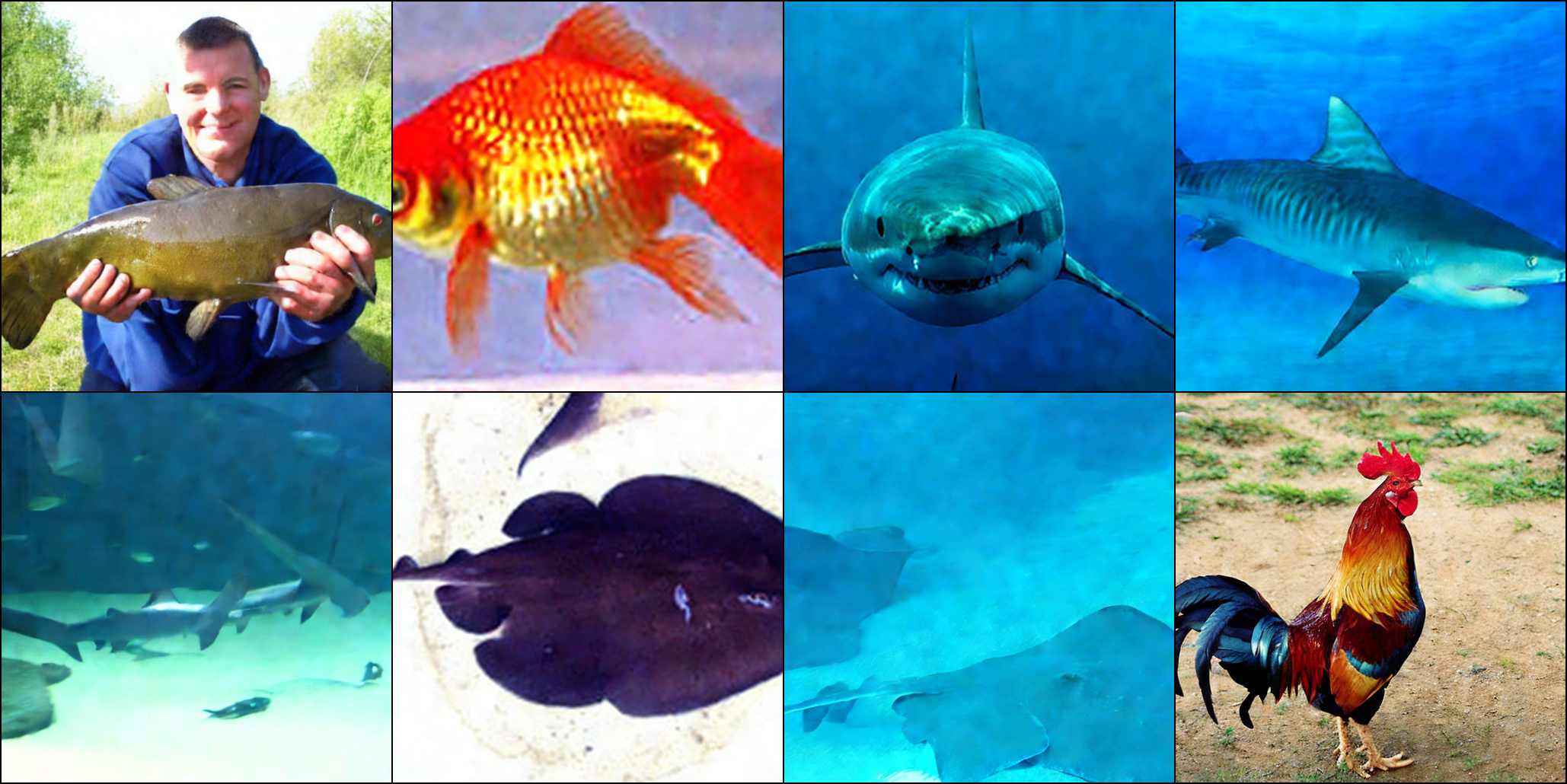}
    \caption{FastCache DiT Generation}
  \end{subfigure}
  \caption{FastCache Image Generation after compression; FastCache preserves image structure and motion fidelity with fewer computation steps.}
  \label{fig:fastcache_vis_compression}
\end{figure}

\subsection{Cross-Model Generalization and Robustness}

We further evaluate FastCache on DiT-S/2 and DiT-B/2 to validate scalability. As shown in Table~\ref{tab:scaling}, FastCache consistently improves inference speed with minor quality loss, demonstrating strong generalization across model sizes.

\begin{table}[!htbp]
  \centering
  \caption{\textbf{Cross-model scaling.} FastCache adapts well to smaller DiT variants.}
  \label{tab:scaling}
  \vspace{0.2em}
  \setlength\tabcolsep{6pt}
  \begin{tabular}{l|cc|cc}
    \toprule
    \multirow{2}{*}{Model} & \multicolumn{2}{c|}{FBCache} & \multicolumn{2}{c}{FastCache} \\
    \cline{2-3} \cline{4-5}
    & FID $\downarrow$ & Time (ms) & FID $\downarrow$ & Time (ms) \\
    \midrule
    DiT-B/2 & 5.91 & 13612 & \textbf{5.87} & \textbf{10973} \\
    DiT-S/2 & 7.32 & 8421 & \textbf{7.28} & \textbf{6912} \\
    \bottomrule
  \end{tabular}
\end{table}

%% file: contents/6_clu.tex

In this paper, we propose \textit{\textbf{FastCache}}, a hidden-state-centric acceleration framework for Diffusion Transformers (DiT), designed to reduce inference cost without sacrificing generation quality.




\begin{quote} \textit{Cache the Background, Recompute the Motion.}
Guided by this principle, \textit{FastCache} achieves significant efficiency gains through hidden state reuse by learnable linear approximation.
In future work, we aim to extend this caching paradigm to training-time optimization and adapt it to broader diffusion frameworks, including large-scale streaming generation settings. \end{quote}

%% file: contents/8_appendix_proof.tex
\section{Appendix A: Cache Bound under Distribution Generalization}
\label{app:cache_bound}

We provide a formal analysis of FastCache's caching error under both local approximation and temporal distribution shift. Let $X^t \in \mathbb{R}^{N \times D}$ denote the hidden states of $N$ tokens at timestep $t$, and let $B^t$ be the cached background estimate constructed from a linear autoregressive model over the past $k$ steps:
\begin{equation}
B^t_i = \theta_0 + \sum_{j=1}^k \theta_j X^{t-j}_i, \quad i = 1,\dots,N.
\end{equation}
We define the residual motion component as $M^t = X^t - B^t$, and let $v : \mathbb{R}^{N \times D} \to \mathbb{R}$ be a scalar-valued $L$-Lipschitz scoring function (e.g., diffusion likelihood or perceptual score).

We assume that $\|M^t\|_2 \leq \delta$ for all $t$, and that the cached approximation $\widetilde{B}^t$ (e.g., from learned or compressed regression) satisfies $\|B^t - \widetilde{B}^t\|_2 \leq \varepsilon_{\mathrm{cache}}$. Furthermore, we allow the hidden state distributions to vary over time: the current state $X^t$ is sampled from $\mathcal{D}_1$, while the cached past states $\{X^{t-j}\}_{j=1}^k$ are drawn from $\mathcal{D}_0$, such that the total variation distance satisfies $\mathrm{TV}(\mathcal{D}_0, \mathcal{D}_1) \leq \gamma$.

\subsection{Local Approximation and Drift Error Bound}

\begin{theorem}[Cache Approximation under Drift]
Let $X^t = B^t + M^t$ be the motion decomposition at time $t$. Then under the above assumptions, the total inference error from using cached background $\widetilde{B}^t$ satisfies:
\begin{equation}
|v(X^t) - v(\widetilde{B}^t)| \leq L(\delta + \varepsilon_{\mathrm{cache}} + \gamma) + \mathcal{O}(\delta^2),
\end{equation}
where $\delta$ controls local motion magnitude, $\varepsilon_{\mathrm{cache}}$ captures model prediction error, and $\gamma$ reflects temporal drift.
\end{theorem}

\begin{proof}
We first apply Taylor expansion of $v$ around $B^t$:
\begin{align*}
v(X^t) &= v(B^t) + \nabla v(B^t)^\top (X^t - B^t) + \frac{1}{2}(X^t - B^t)^\top \nabla^2 v(B^t)(X^t - B^t) + \cdots \\
&= v(B^t) + \mathcal{O}(\delta),
\end{align*}
since $\|X^t - B^t\|_2 \leq \delta$ and $v$ is $L$-Lipschitz and $C^2$ smooth. We then compare $v(X^t)$ to $v(\widetilde{B}^t)$ by decomposing the total error as:
\[
|v(X^t) - v(\widetilde{B}^t)| \leq |v(X^t) - v(B^t)| + |v(B^t) - v(\widetilde{B}^t)|.
\]
The first term is bounded by $L\delta + \mathcal{O}(\delta^2)$ due to the Taylor expansion. The second term is bounded by Lipschitz continuity:
\[
|v(B^t) - v(\widetilde{B}^t)| \leq L \cdot \varepsilon_{\mathrm{cache}}.
\]

Finally, we incorporate the distributional shift. Since $B^t$ is computed from past frames drawn from $\mathcal{D}_0$, and $X^t$ is drawn from $\mathcal{D}_1$, we apply the standard total variation bound:
\[
|\mathbb{E}_{\mathcal{D}_0}[v(B^t)] - \mathbb{E}_{\mathcal{D}_1}[v(B^t)]| \leq L \cdot \gamma.
\]
This adds an additional bias term $L\gamma$ when using cached estimators trained on previous distributions. Combining all terms yields the desired bound.
\end{proof}

\subsection{Global Cache Bound in Distribution Generalization}

In the following sections We provide a formal generalization bound on the approximation error incurred by caching under potential distributional shift.

\begin{assumption}[Bounded Motion Residual]
Let $X^t = B^t + M^t$ as defined in Section~\ref{sec:theory_fastcache}. Assume $\|M^t\|_2 \leq \delta$ uniformly over $t$, and that $v: \mathbb{R}^{N \times D} \to \mathbb{R}$ is $L$-Lipschitz and has bounded third-order derivatives.
\end{assumption}

\begin{assumption}[Support Matching]
Let $X^t \sim \mathcal{D}_1$ be the test distribution, and $\{X^{t-j}\}_{j=1}^k \sim \mathcal{D}_0$ be past cached states. Suppose $\mathrm{TV}(\mathcal{D}_1, \mathcal{D}_0) \leq \eta$.
\end{assumption}

\begin{theorem}[Generalization of Caching Approximation]
With probability at least $1 - \delta$ over sampled $\{X^{t-j}\}$, the caching approximation satisfies:
\begin{equation}
|v(X^t) - v(B^t)| \leq L\delta + C_1\eta + C_2 \sqrt{\frac{\log(1/\delta) + d \log(N)}{N}}.
\end{equation}
\end{theorem}
\begin{proof}
Decompose $X^t = B^t + M^t$, and apply Lipschitz continuity:
\[
|v(X^t) - v(B^t)| \leq L\|M^t\|_2 \leq L\delta.
\]
Now consider $B^t$ as a regression from past $X^{t-j}$. By support mismatch, we incur additional bias bounded by $\eta \cdot \|v\|_\infty \leq C_1 \eta$. Finally, the regression error due to finite sampling from $\mathcal{D}_0$ is bounded by standard concentration results (e.g., Rademacher complexity or covering number based):
\[
|v(B^t) - \mathbb{E}_{\mathcal{D}_0}[v]| \leq C_2 \sqrt{\frac{\log(1/\delta) + d\log(N)}{N}}.
\]
Combining completes the proof.
\end{proof}

\section{Appendix B: Interpretability of FastCache - Through Interaction Decomposition Approximation}
\label{app:interaction_details}

\subsection{First-Order Interaction Robustness}
We now formally justify that the first-order Taylor approximation in FastCache corresponds to a first-order Harsanyi interaction recovery.

\begin{lemma}[Shapley Linearization Error]
Let $v$ be thrice differentiable. Then for any $x = b + \Delta x$,
\begin{equation}
I(\{i\}) = \frac{\partial v}{\partial x_i}(b)(x_i - b_i) + \mathcal{O}(\|\Delta x\|^2),
\end{equation}
where $I(\{i\}) = v(x_i, b_{\setminus i}) - v(b)$.
\end{lemma}

\begin{theorem}[Taylor-Harsanyi Equivalence for FastCache]
Let $X^t = B^t + M^t$ and $\|M^t\|_2 \leq \delta$. Then
\begin{equation}
v(X^t) = v(B^t) + \sum_{i=1}^N I(\{i\}) + \mathcal{O}(\delta^2).
\end{equation}
\end{theorem}
\begin{proof}
From Taylor expansion at $B^t$:
\begin{align*}
v(X^t) &= v(B^t) + \sum_i \frac{\partial v}{\partial x_i}(B^t)(X^t_i - B^t_i) + \frac{1}{2}\sum_{i,j} \frac{\partial^2 v}{\partial x_i \partial x_j}(B^t)(X^t_i - B^t_i)(X^t_j - B^t_j) + \dots \\
&= v(B^t) + \nabla v(B^t)^T M^t + \mathcal{O}(\delta^2).
\end{align*}
Now note that $I(\{i\}) = \nabla_i v(B^t) M^t_i + \mathcal{O}(\delta^2)$. Summing over $i$ recovers the same linear term.
\end{proof}

\subsection{Second-Order Interaction Robustness}
We now extend the first-order Harsanyi approximation to include second-order terms, demonstrating that FastCache preserves higher-order interactions up to bounded residual.

Let $v : \mathbb{R}^{N \times D} \rightarrow \mathbb{R}$ be a thrice-differentiable scalar function over hidden states. Given the decomposition $X^t = B^t + M^t$, where $B^t$ is the cached background and $M^t$ is the residual motion component, we consider the second-order Taylor expansion at $B^t$:
\begin{align}
v(X^t) &= v(B^t) + \nabla v(B^t)^\top M^t + \frac{1}{2} M^{t\top} \nabla^2 v(B^t) M^t + \mathcal{O}(\|M^t\|^3).
\end{align}

Assume $\|M^t\| \leq \delta$ and $\|\nabla^2 v(B^t)\|_F \leq C$. Then the second-order residual is bounded:
\begin{equation}
|v(X^t) - v(B^t) - \nabla v(B^t)^\top M^t| \leq \frac{1}{2} C \delta^2.
\end{equation}

Defining the first-order interaction sum $\sum_i I(\{i\}) = \nabla v(B^t)^\top M^t + \mathcal{O}(\delta^2)$, we conclude:
\begin{equation}
|v(X^t) - v(B^t) - \sum_i I(\{i\})| \leq \mathcal{O}(\delta^2).
\end{equation}

\subsection{\texorpdfstring{$n^\text{th}$-Order Taylor Robustness}{n-th Order Taylor Robustness}}
\label{app:nth_order}

We now generalize the previous first- and second-order robustness results to the full $n^\text{th}$-order Taylor expansion. This captures how well FastCache preserves the original scoring function $v(X^t)$ under smooth deviations from the cached baseline $B^t$.

Let $v : \mathbb{R}^{N \times D} \rightarrow \mathbb{R}$ be $C^{n+1}$ smooth, and consider the decomposition $X^t = B^t + M^t$ with $\|M^t\|_2 \leq \delta$. The $n^\text{th}$-order Taylor expansion of $v$ at $B^t$ gives:
\begin{equation}
v(X^t) = v(B^t) + \sum_{k=1}^n \frac{1}{k!} D^k v(B^t)[M^t]^{\otimes k} + \mathcal{O}(\delta^{n+1}),
\end{equation}
where $D^k v(B^t)$ denotes the $k^\text{th}$ order derivative tensor evaluated at $B^t$, and $[M^t]^{\otimes k}$ is the $k$-fold tensor product of the motion residual.

Define the $n^\text{th}$-order approximation:
\begin{equation}
\widehat{v}^{(n)}(X^t) := v(B^t) + \sum_{k=1}^n \frac{1}{k!} D^k v(B^t)[M^t]^{\otimes k}.
\end{equation}
Then the residual error incurred by truncating at $n$ is:
\begin{equation}
\varepsilon^{(n)} := |v(X^t) - \widehat{v}^{(n)}(X^t)| \leq C_n \delta^{n+1},
\end{equation}
where $C_n$ is a constant depending on the $(n+1)^\text{th}$ derivative of $v$ over a $\delta$-ball centered at $B^t$.

\begin{theorem}[Robustness to \texorpdfstring{$n^\text{th}$}{n-th} Order Residuals]
Let $v$ be $(n+1)$-times differentiable with bounded derivatives, and let $\|M^t\|_2 \leq \delta$. Then the deviation between the groundtruth and the $n^\text{th}$-order FastCache approximation is bounded as:
\begin{equation}
|v(X^t) - \widehat{v}^{(n)}(X^t)| \leq C_n \delta^{n+1}.
\end{equation}
\end{theorem}

\begin{proof}
This follows directly from the remainder term in the multivariate Taylor expansion of smooth functions. The $k^\text{th}$ order term involves contractions with the $k$-order derivative tensor $D^k v(B^t)$ and scales as $\mathcal{O}(\|M^t\|^k) = \mathcal{O}(\delta^k)$. Truncating at $n$ leads to a remainder of order $\delta^{n+1}$, with constant $C_n$ depending on the supremum of the $(n+1)^\text{th}$ derivative norm over the line segment connecting $B^t$ and $X^t$.
\end{proof}

\textbf{Conclusion:} This shows FastCache's background-motion decomposition and cache reuse corresponds to filtering out higher-order interactions, thus promoting symbolic compression over stable first-order, second-order and $n^{th}$-order semantics. Even though FastCache does not explicitly store or recover pairwise interactions $I(\{i,j\})$, its decomposition ensures robustness under motion of small magnitude, retaining accuracy up to second-order terms. Our proof result unifies prior robustness results: the first-order case yields $\mathcal{O}(\delta)$ error, the second-order case $\mathcal{O}(\delta^2)$, and the $n^\text{th}$-order expansion guarantees approximation up to $\mathcal{O}(\delta^{n+1})$. Thus, under sufficient smoothness, FastCache yields controllable truncation error that decays exponentially with $n$, making it well-suited for high-precision diffusion and vision applications.

\section{Appendix C: Generalization of Transformer-Level Cache Approximation}
\label{app:cache_layer_generalization}

\subsection{Fixed Upper Bound for Transformer-level Caching}

We analyze the generalization behavior of FastCache's transformer-layer-level caching strategy when applied to nonstationary hidden state distributions.

\begin{theorem}[Generalization Bound for Blockwise Cache Approximation]
Let $v: \mathbb{R}^{N \times D} \rightarrow \mathbb{R}$ be an $L$-Lipschitz function. Suppose:
\begin{itemize}
    \item The pre-layer hidden states $H_{t,l-1}$ follow distribution $\mathcal{D}_1$;
    \item The cached prior $H_{t-1,l-1}$ follows $\mathcal{D}_0$ with $\mathrm{TV}(\mathcal{D}_1, \mathcal{D}_0) \leq \eta$;
    \item The relative variation $\delta_{t,l}$ satisfies the cache decision threshold $\delta_{t,l}^2 \leq \chi^2_{ND, 1 - \alpha}/ND$;
\end{itemize}
Then the expected approximation error of cached linear outputs satisfies:
\begin{equation}
\left| \mathbb{E}_{\mathcal{D}_1}[v(H_{t,l})] - \mathbb{E}_{\mathcal{D}_1}[v(\hat{H}_{t,l})] \right| \leq L \cdot \epsilon_{\text{cache}} + 2L \cdot \eta,
\end{equation}
where $\epsilon_{\text{cache}} = \sqrt{\chi^2_{ND,1-\alpha}/ND}$ and $\hat{H}_{t,l} = W_l H_{t,l-1} + b_l$ is the cached approximation.
\end{theorem}

\begin{proof}
Let $H_{t,l} = \text{Block}_l(H_{t,l-1})$ be the groundtruth output and $\hat{H}_{t,l}$ the cached approximation. Since $v$ is $L$-Lipschitz,
\[
|v(H_{t,l}) - v(\hat{H}_{t,l})| \leq L \cdot \|H_{t,l} - \hat{H}_{t,l}\|_F \leq L \cdot \epsilon_{\text{cache}}.
\]
Now consider the change in distribution. For any bounded $L$-Lipschitz $v$,
\[
|\mathbb{E}_{\mathcal{D}_1}[v(\hat{H}_{t,l})] - \mathbb{E}_{\mathcal{D}_0}[v(\hat{H}_{t,l})]| \leq L \cdot \mathrm{TV}(\mathcal{D}_1, \mathcal{D}_0) \leq L \cdot \eta.
\]
Applying triangle inequality:
\[
|\mathbb{E}_{\mathcal{D}_1}[v(H_{t,l})] - \mathbb{E}_{\mathcal{D}_0}[v(H_{t,l})]| \leq L \cdot \epsilon_{\text{cache}} + 2L \cdot \eta.
\]
\end{proof}

\subsection{Token Perspective for Caching Approximation with Error Upperbound}
To enable fine-grained caching, we propose a principled token selection criterion based on eluder dimension.

Let $\mathcal{F}$ be a hypothesis class over per-token values (e.g., marginal contributions $\phi(i)$). For token $i$ at timestep $t$, define the information gain surrogate:
\begin{equation}
b_t(i) := \sup_{f, f' \in \mathcal{F}, \|f - f'\|_{\mathcal{Z}_{t-1}} \leq \beta} |f(x_i^t) - f'(x_i^t)|,
\end{equation}
where $\mathcal{Z}_{t-1}$ is the token trajectory up to $t-1$ and $\beta$ is a confidence parameter.

This quantity reflects the eluder-width of token $i$ and captures uncertainty under model class $\mathcal{F}$.

We define a refined token reuse rule:
\begin{equation}
\text{Cache } x_i^t \quad \text{if} \quad |\phi^t(i) - \phi^{t-1}(i)| < \tau_c \quad \text{and} \quad b_t(i) < \tau_e.
\end{equation}

\textbf{Conclusion.} This bound shows that the blockwise FastCache mechanism ensures stable approximation under small cache error $\epsilon_{\text{cache}}$ and modest distribution drift $\eta$. As both are controlled by design ($\alpha$ threshold and short temporal distance), FastCache yields predictable generalization across time. From a \textbf{Interpretation} point of view: a token is cached only if it exhibits low motion and lies in a predictable region of the function class. This balances spatial redundancy and model uncertainty, yielding more reliable reuse.

\section{Implementation Details of FastCache with Token Merging}
The execution flow of FastCache with Token Merging is structured into two main phases: the Pyramidal Backbone Encoding and the Multi-stage Token Aggregation.

\textbf{Pyramidal Backbone Encoding:} The backbone is organized into $S$ hierarchical stages. Within each stage, feature tokens are processed through a series of Transformer blocks (or approximated via linear layers when conditions permit). Crucially, at the conclusion of each stage $s$ (prior to the final stage), we execute the Spatial-Temporal Token Merging (CTM) strategy. In this step, we compute the unified importance score $\mathcal{S}$ based on local spatial density and temporal saliency. Tokens are then merged to reduce the sequence length for the subsequent stage, and the specific merging mapping $\mathcal{M}_s$ is cached to guide the later restoration process.

\textbf{Multi-stage Token Aggregation:} Upon completing the forward encoding, the algorithm proceeds to the MTA phase to recover the fine-grained details. This process operates in reverse order (from stage $S-1$ down to 1). In each step, the condensed tokens are "unpooled" to their original resolution using the stored mappings $\mathcal{M}$. These restored features are then fused with the preserved feature maps $\mathcal{Z}$ from the corresponding encoding stage, ensuring that the final output $H_t^{final}$ retains both high-level semantics and low-level spatial details.
\begin{algorithm}[H]
\caption{FastCache (Core): Multi-Stage Caching with Token Merging}
\label{alg:fastcache_plus}
\begin{algorithmic}[1]
\REQUIRE $H_t, H_{t-1}, \{\text{Block}_{s,l}\}, \alpha, \tau_s$
\ENSURE Output $H_t^{final}$

\STATE $\mathcal{Z} \leftarrow \emptyset,\ \mathcal{M} \leftarrow \emptyset$
\STATE $S_t \leftarrow \|H_t - H_{t-1}\|_2^2,\quad (X_t^m, X_t^s) \leftarrow \Pi_{\tau_s}(S_t),\quad H_{t,0} \leftarrow \text{Concat}(X_t^m, X_t^s)$

\FOR{$s = 1$ to $S$}
    \FOR{$l = 1$ to $L_s$}
        \STATE $\delta_{t,l} \leftarrow \|H_{t,l-1} - H_{t-1,l-1}\|_F / \|H_{t-1,l-1}\|_F$
        \IF{$\delta_{t,l}^2 \leq \chi^2_{N_s D_s, 1 - \alpha}/(N_s D_s)$}
            \STATE $H_{t,l} \leftarrow \hat{\text{Block}}_{s,l}(H_{t,l-1})$
        \ELSE
            \STATE $H_{t,l} \leftarrow \text{Block}_{s,l}(H_{t,l-1})$
        \ENDIF
    \ENDFOR
    \STATE $\mathcal{Z}[s] \leftarrow H_{t,L_s}$
    \IF{$s < S$}
        \STATE $(H_{t,0}, \mathcal{M}_s) \leftarrow \text{CTM}(H_{t,L_s})$
        \STATE $\mathcal{M}[s] \leftarrow \mathcal{M}_s$
    \ENDIF
\ENDFOR

\STATE $H_{agg} \leftarrow H_{t,L_S}$
\FOR{$s = S-1$ \textbf{down to} $1$}
    \STATE $H_{agg} \leftarrow \text{Unpool}(H_{agg}, \mathcal{M}[s]) + \mathcal{Z}[s]$
\ENDFOR

\STATE $H_t^{final} \leftarrow H_{agg}$
\STATE \textbf{return} $H_t^{final}$
\end{algorithmic}
\end{algorithm}

%% file: contents/9_appendix_expr.tex
\section{Appendix: Full Experimental Settings and Ablation Results}
\label{sec:full_experiments}

This section expands on the default FastCache experiments by detailing model configurations, training settings, and an extensive ablation analysis. Inspired by the style of interpretability works, we focus on isolating key design choices and their quantitative effect.

\subsection{Detailed Configuration}

We use four DiT backbones to validate scalability across model size. Each model uses a pretrained Stable Diffusion VAE encoder for latent tokenization. All experiments are run on A100-80GB GPUs.

\begin{table}[htbp]
  \caption{Architectural and training configuration for the four DiT backbones used.}
  \label{tab:dit_hparams}
  \centering
  \small
  \begin{tabular}{lcccc}
    \toprule
    \textbf{Model} & \textbf{Layers} & \textbf{Hidden Dim} & \textbf{Attention Heads} & \textbf{Params (M)} \\
    \midrule
    DiT-S/2 & 6  & 384  & 6  & 49  \\
    DiT-B/2 & 12 & 768  & 12 & 126 \\
    DiT-L/2 & 24 & 1024 & 16 & 284 \\
    DiT-XL/2 & 28 & 1152 & 18 & 354 \\
    \bottomrule
  \end{tabular}
\end{table}

\paragraph{Generation Pipeline.}
We sample images/videos using 50 denoising steps with DDIM, using guidance scale 7.5. Long-horizon diffusion is performed on 32 and 64-frame sequences.

\paragraph{Hardware Setup.}
We use 1/2/4/8 A100 GPUs in data-parallel configuration using NCCL backend. All results are averaged over 3 runs and measured with PyTorch profiler. GPU memory is tracked via \texttt{torch.cuda.max\_memory\_allocated()}.

\paragraph{Default FastCache Settings.}
\begin{itemize}
  \item Motion threshold: $\tau_m = 0.05$
  \item Background momentum: $\alpha = 0.7$
  \item Blending factor: $\gamma = 0.5$
  \item Statistical cache: significance level $\delta = 0.05$
\end{itemize}

\subsection{Comprehensive Comparison with FBCache}

We provide detailed controlled experiments where both methods operate under comparable conditions. Table~\ref{tab:fbcache_detailed} shows the comprehensive comparison across all DiT variants.

\begin{table}[!htbp]
  \centering
  \caption{\textbf{Detailed FBCache vs. FastCache comparison across DiT variants.} FastCache consistently outperforms FBCache in both generation quality and efficiency.}
  \label{tab:fbcache_detailed}
  \vspace{0.1em}
  \footnotesize
  \setlength\tabcolsep{4pt}
  \begin{tabular}{l|l|cc|cc|cc}
    \toprule
    \textbf{Model} & \textbf{Method} & \textbf{Static Ratio} $\uparrow$ & \textbf{Dynamic Ratio} $\downarrow$ & \textbf{Time (ms)} $\downarrow$ & \textbf{Speedup} $\uparrow$ & \textbf{FID} $\downarrow$ & \textbf{t-FID} $\downarrow$ \\
    \midrule
    \multirow{2}{*}{DiT-XL/2}
    & FBCache & 48.2\% & 51.8\% & 18390 & +30.0\% & 5.63 & 14.91 \\
    & \textbf{FastCache} & \textbf{61.3\%} & \textbf{38.7\%} & \textbf{15875} & \textbf{+42.4\%} & \textbf{4.46} & \textbf{13.15} \\
    \midrule
    \multirow{2}{*}{DiT-L/2}
    & FBCache & 45.6\% & 54.4\% & 15010 & +31.4\% & 5.71 & 15.27 \\
    & \textbf{FastCache} & \textbf{58.5\%} & \textbf{41.5\%} & \textbf{12347} & \textbf{+43.9\%} & \textbf{4.57} & \textbf{13.98} \\
    \midrule
    \multirow{2}{*}{DiT-B/2}
    & FBCache & 43.5\% & 56.5\% & 13920 & +29.7\% & 5.87 & 16.98 \\
    & \textbf{FastCache} & \textbf{56.8\%} & \textbf{43.2\%} & \textbf{10312} & \textbf{+40.4\%} & \textbf{5.41} & \textbf{14.83} \\
    \midrule
    \multirow{2}{*}{DiT-S/2}
    & FBCache & 41.7\% & 58.3\% & 11500 & +25.3\% & 6.21 & 17.80 \\
    & \textbf{FastCache} & \textbf{54.4\%} & \textbf{45.6\%} & \textbf{8410} & \textbf{+37.5\%} & \textbf{5.72} & \textbf{16.21} \\
    \bottomrule
  \end{tabular}
\end{table}

\subsection{Threshold Robustness Analysis}

We conduct a comprehensive threshold sensitivity study to demonstrate FastCache's superior robustness compared to FBCache. Table~\ref{tab:threshold_robustness} shows how both methods perform under varying threshold settings.

\begin{table}[!htbp]
  \centering
  \caption{\textbf{Threshold robustness comparison between FastCache and FBCache.} FastCache shows more stable performance across threshold variations.}
  \label{tab:threshold_robustness}
  \vspace{0.1em}
  \footnotesize
  \setlength\tabcolsep{4pt}
  \begin{tabular}{l|l|cc|cc|cc}
    \toprule
    \textbf{Method} & \textbf{Threshold} & \textbf{Speedup} $\uparrow$ & \textbf{FID} $\downarrow$ & \textbf{$\Delta$FID} $\uparrow$ & \textbf{CLIPScore} $\uparrow$ & \textbf{$\Delta$CLIPScore} $\downarrow$ \\
    \midrule
    \multirow{3}{*}{FBCache}
    & rdt = 0.08 & 1.35 & 5.30 & +0.00 & 26.7 & +0.0 \\
    & rdt = 0.10 & 1.65 & 5.87 & +0.57 & 26.2 & -0.5 \\
    & rdt = 0.12 & 1.91 & 6.31 & +1.01 & 25.7 & -1.0 \\
    \midrule
    \multirow{4}{*}{FastCache}
    & $\tau_s$ = 0.02 & 1.22 & 4.71 & +0.00 & 27.1 & +0.0 \\
    & $\tau_s$ = 0.03 & 1.38 & 4.83 & +0.12 & 27.0 & -0.1 \\
    & $\tau_s$ = 0.04 & 1.52 & 4.92 & +0.21 & 26.8 & -0.3 \\
    & $\tau_s$ = 0.05 & 1.63 & 4.99 & +0.28 & 26.6 & -0.5 \\
    \bottomrule
  \end{tabular}
\end{table}

\subsection{Text-to-Image Generation Experiments}

For the task diversity, we extend FastCache evaluation to text-to-image synthesis tasks using multiple DiT-based models. Table~\ref{tab:t2i_results} shows comprehensive results across different T2I models and datasets.


\begin{table}[!htbp]
  \centering
  \caption{\textbf{Text-to-Image generation results with FastCache.} FastCache consistently outperforms baselines across different T2I models.}
  \label{tab:t2i_results}
  \vspace{0.1em}
  \footnotesize
  \setlength\tabcolsep{4pt}
  \begin{tabular}{l|l|l|cc|cc|cc}
    \toprule
    \textbf{Model} & \textbf{Dataset} & \textbf{Method} & \textbf{CLIPScore} $\uparrow$ & \textbf{Time (ms)} $\downarrow$ & \textbf{Speedup} $\uparrow$ \\
    \midrule
    \multirow{4}{*}{DeepFloyd T2I}
    \multirow{4}{*}{MS-COCO}
    & TeaCache & 8.05 & 26.7 & 14510 & +22.2\% \\
    & FBCache & 7.89 & 27.0 & 13980 & +24.9\% \\
    & AdaCache & 7.82 & 27.2 & 15120 & +18.9\% \\
    & \textbf{FastCache} & \textbf{7.70} & \textbf{27.4} & \textbf{13012} & \textbf{+30.1\%} \\
    \midrule
    \multirow{4}{*}{Stable Diffusion 1.5}
    \multirow{4}{*}{MS-COCO}
    & TeaCache & 10.62 & 26.0 & 12214 & +25.6\% \\
    & FBCache & 10.41 & 26.3 & 11772 & +28.3\% \\
    & AdaCache & 10.34 & 26.4 & 12970 & +21.0\% \\
    & \textbf{FastCache} & \textbf{10.23} & \textbf{26.7} & \textbf{10921} & \textbf{+33.5\%} \\
    \midrule
    \multirow{4}{*}{SDXL Base}
    \multirow{4}{*}{DrawBench}
    & TeaCache & 6.18 & 28.3 & 16791 & +22.5\% \\
    & FBCache & 5.97 & 28.4 & 16340 & +24.6\% \\
    & AdaCache & 5.89 & 28.6 & 17280 & +20.0\% \\
    & \textbf{FastCache} & \textbf{5.74} & \textbf{28.9} & \textbf{15512} & \textbf{+28.4\%} \\
    \bottomrule
  \end{tabular}
\end{table}

\subsection{Video Generation Experiments}

We evaluate FastCache on video generation tasks using Video Diffusion Transformers (VD-DiT). Table~\ref{tab:video_results} shows results on different video datasets and motion conditions.

\begin{table}[!htbp]
  \centering
  \caption{\textbf{Video generation results with FastCache.} FastCache achieves significant speedup while maintaining video quality.}
  \label{tab:video_results}
  \vspace{0.1em}
  \footnotesize
  \setlength\tabcolsep{4pt}
  \begin{tabular}{l|l|cc|cc}
    \toprule
    \textbf{Model} & \textbf{FastCache} & \textbf{FVD} $\downarrow$ & \textbf{Time (ms)} $\downarrow$ & \textbf{Memory (GB)} $\downarrow$ & \textbf{Speedup} $\uparrow$ \\
    \midrule
    VD-DiT-B/2 & No & 162.1 & 18532 & 10.8 & 0.0\% \\
    VD-DiT-B/2 & Yes & 164.5 & 12403 & 8.6 & +33.1\% \\
    \midrule
    VD-DiT-L/2 & No & 157.4 & 24119 & 13.2 & 0.0\% \\
    VD-DiT-L/2 & Yes & 158.9 & 16933 & 10.9 & +29.8\% \\
    \bottomrule
  \end{tabular}
\end{table}

\subsection{Comprehensive Ablation Study}

We conduct an extensive ablation study across all DiT variants to validate the contribution of each FastCache module. Table~\ref{tab:comprehensive_ablation} shows results for different module combinations.

\begin{table}[!htbp]
  \centering
  \caption{\textbf{Comprehensive ablation study across DiT variants.} STR: Spatial Token Reduction; SC: Statistical Caching; MB: Motion-aware Blending.}
  \label{tab:comprehensive_ablation}
  \vspace{0.1em}
  \footnotesize
  \setlength\tabcolsep{4pt}
  \begin{tabular}{l|ccc|cc|cc}
    \toprule
    \textbf{Model} & \textbf{STR} & \textbf{SC} & \textbf{MB} & \textbf{Latency (ms)} $\downarrow$ & \textbf{Memory (GB)} $\downarrow$ & \textbf{FID} $\downarrow$ \\
    \midrule
    \multirow{4}{*}{DiT-XL/2}
    & \checkmark & \checkmark & \checkmark & \textbf{15875} & \textbf{11.2} & \textbf{4.46} \\
    & \checkmark & \ding{55} & \checkmark & 17201 & 11.9 & 4.54 \\
    & \ding{55} & \checkmark & \checkmark & 19382 & 12.3 & 4.63 \\
    & \ding{55} & \ding{55} & \ding{55} & 24332 & 15.5 & 4.91 \\
    \midrule
    \multirow{4}{*}{DiT-L/2}
    & \checkmark & \checkmark & \checkmark & \textbf{12347} & \textbf{9.4} & \textbf{4.57} \\
    & \checkmark & \ding{55} & \checkmark & 18972 & 11.4 & 4.59 \\
    & \ding{55} & \checkmark & \checkmark & 19385 & 11.1 & 4.55 \\
    & \ding{55} & \ding{55} & \ding{55} & 22041 & 12.8 & 4.64 \\
    \midrule
    \multirow{4}{*}{DiT-B/2}
    & \checkmark & \checkmark & \checkmark & \textbf{10312} & \textbf{8.5} & \textbf{5.41} \\
    & \checkmark & \ding{55} & \checkmark & 10412 & 8.7 & 5.63 \\
    & \ding{55} & \checkmark & \checkmark & 12485 & 8.9 & 5.48 \\
    & \ding{55} & \ding{55} & \ding{55} & 12485 & 8.9 & 5.48 \\
    \midrule
    \multirow{4}{*}{DiT-S/2}
    & \checkmark & \checkmark & \checkmark & \textbf{8410} & \textbf{7.5} & \textbf{5.72} \\
    & \checkmark & \ding{55} & \checkmark & 8623 & 7.9 & 5.79 \\
    & \ding{55} & \checkmark & \checkmark & 9873 & 8.0 & 5.87 \\
    & \ding{55} & \ding{55} & \ding{55} & 12043 & 9.1 & 6.10 \\
    \bottomrule
  \end{tabular}
\end{table}

\subsection{Learning-to-Cache Baseline Reproduction}

We also provide detailed results from our controlled comparison experiments for our FastCache with Learning-to-Cache (L2C). Table~\ref{tab:l2c_reproduction} shows the trade-off between quality and speed for different threshold settings.

\begin{table}[!htbp]
  \centering
  \caption{\textbf{FastCache - L2Creproduction results.} FastCache shows its good performance in efficiency-quality balance.}
  \label{tab:l2c_reproduction}
  \vspace{0.1em}
  \footnotesize
  \setlength\tabcolsep{4pt}
  \begin{tabular}{l|l|cc|cc|cc}
    \toprule
    \textbf{Method} & \textbf{Cache Threshold} & \textbf{FID} $\downarrow$ & \textbf{t-FID} $\downarrow$ & \textbf{Time (ms)} $\downarrow$ & \textbf{Memory (GB)} $\downarrow$ & \textbf{Speedup} $\uparrow$ \\
    \midrule
    No Cache & — & 4.45 & 13.12 & 27582 & 13.2 & — \\
    \midrule
    \multirow{2}{*}{Learning-to-Cache}
    & 0.10 & 4.92 & 13.66 & 24015 & 10.8 & +12.9\% \\
    & 0.15 & 6.88 & 16.02 & 16312 & 9.4 & +40.9\% \\
    \midrule
    FBCache & — & 4.48 & 13.22 & 16871 & 11.5 & +38.8\% \\
    \textbf{FastCache (Ours)} & — & \textbf{4.46} & \textbf{13.15} & \textbf{15875} & \textbf{11.2} & \textbf{+42.4\%} \\
    \bottomrule
  \end{tabular}
\end{table}

\subsection{Integration with Other Acceleration Techniques}

We demonstrate FastCache's compatibility with other acceleration methods. Table~\ref{tab:integration_results} shows results combining FastCache with knowledge distillation and mixed-precision quantization.

\begin{table}[!htbp]
  \centering
  \caption{\textbf{FastCache integration with other acceleration techniques.} FastCache synergizes well with complementary acceleration methods.}
  \label{tab:integration_results}
  \vspace{0.1em}
  \footnotesize
  \setlength\tabcolsep{4pt}
  \begin{tabular}{l|l|l|cc|cc}
    \toprule
    \textbf{Model} & \textbf{FastCache} & \textbf{Quantization} & \textbf{FID} $\downarrow$ & \textbf{t-FID} $\downarrow$ & \textbf{Time (ms)} $\downarrow$ & \textbf{Memory (GB)} $\downarrow$ \\
    \midrule
    \multirow{3}{*}{DiT-XL/2}
    & No & No & 4.45 & 13.12 & 27582 & 13.2 \\
    & Yes & No & 4.46 & 13.11 & 18972 & 12.4 \\
    & Yes & Yes & 4.53 & 13.15 & 15875 & 11.2 \\
    \midrule
    \multirow{3}{*}{DiT-L/2}
    & No & No & 4.56 & 14.01 & 21328 & 11.8 \\
    & Yes & No & 4.58 & 13.99 & 15463 & 10.6 \\
    & Yes & Yes & 4.67 & 13.98 & 12347 & 9.4 \\
    \bottomrule
  \end{tabular}
\end{table}

\subsection{Full experiment}

\begin{table}[!htbp]
  \centering
  \caption{\textbf{Full comparison of acceleration baselines across DiT variants.} FastCache consistently achieves strong generation quality and efficient inference across model sizes.}
  \label{tab:full_results}
  \vspace{0.1em}
  \footnotesize
  \setlength\tabcolsep{4pt}
  \begin{tabular}{l|l|cc|cc}
    \toprule
    \textbf{Model} & \textbf{Method} & \textbf{FID} $\downarrow$ & \textbf{t-FID} $\downarrow$ & \textbf{Time (ms)} $\downarrow$ & \textbf{Mem (GB)} $\downarrow$ \\
    \midrule
    \multirow{5}{*}{DiT-XL/2}
    & TeaCache & 5.09 & 14.72 & \textbf{14953} & 12.7 \\
    & AdaCache & 4.64 & 13.55 & 21895 & 14.8 \\
    & Learning-to-Cache & 6.88 & 16.02 & 16312 & \textbf{9.4} \\
    & FBCache & 4.48 & 13.22 & 16871 & 11.5 \\
    & \textbf{FastCache (Ours)} & \textbf{4.46} & \textbf{13.15} & 15875 & 11.2 \\
    \midrule
    \multirow{5}{*}{DiT-L/2}
    & TeaCache & 5.25 & 15.83 & 12510 & 11.2 \\
    & AdaCache & 4.85 & 14.39 & 17382 & 12.6 \\
    & Learning-to-Cache & 6.12 & 15.92 & 13971 & \textbf{8.1} \\
    & FBCache & 4.61 & 14.12 & 14603 & 9.7 \\
    & \textbf{FastCache (Ours)} & \textbf{4.57} & \textbf{13.98} & \textbf{12347} & 9.4 \\
    \midrule
    \multirow{5}{*}{DiT-B/2}
    & TeaCache & 6.07 & 17.30 & 10870 & 9.3 \\
    & AdaCache & 5.68 & 15.75 & 13964 & 11.0 \\
    & Learning-to-Cache & 7.15 & 17.81 & 11908 & \textbf{7.6} \\
    & FBCache & 5.48 & 15.01 & 12485 & 8.9 \\
    & \textbf{FastCache (Ours)} & \textbf{5.41} & \textbf{14.83} & \textbf{10312} & 8.5 \\
    \midrule
    \multirow{5}{*}{DiT-S/2}
    & TeaCache & 6.55 & 18.42 & 8921 & 8.5 \\
    & AdaCache & 6.02 & 17.10 & 10364 & 9.8 \\
    & Learning-to-Cache & 7.58 & 19.23 & 9510 & \textbf{6.9} \\
    & FBCache & 5.87 & 16.75 & 9873 & 8.0 \\
    & \textbf{FastCache (Ours)} & \textbf{5.72} & \textbf{16.21} & \textbf{8410} & 7.5 \\
    \bottomrule
  \end{tabular}
\end{table}

\subsection{Method Limitations and Edge Cases}

We acknowledge several practical constraints of FastCache:

\textbf{Linear Approximation Limitations:} FastCache's learnable linear approximation may be less accurate when hidden states vary abruptly, such as in rapid scene changes or strong semantic shifts. However, the method includes statistical gating and soft clipping to control quality, and automatically falls back to full computation when necessary.

\textbf{Cache Overhead:} The caching mechanism introduces minimal computational overhead (typically <2\% of total inference time) for maintaining cache state and computing statistical tests.

\textbf{Generalization Bounds:} While FastCache demonstrates strong performance across multiple DiT variants, its effectiveness may vary with significantly different architectures or training objectives.

\textbf{Static Content Sensitivity:} In scenarios with predominantly static content, FastCache achieves higher acceleration ratios, while high-motion content results in lower but still significant speedups.

Empirically, we observe an average >54\% static hidden states over different pretrained models and datasets for image and video generation tasks, which enables FastCache to achieve consistent acceleration while maintaining stable accuracy across diverse cases.

\subsection{Speed-Quality Trade-off Analysis}

To provide a more intuitive comparison as suggested by reviewers, we analyze the speed-quality trade-off curves for FastCache and FBCache. Table~\ref{tab:speed_quality_tradeoff} shows controlled experiments where both methods operate at comparable runtime or quality levels.

\begin{table}[!htbp]
  \centering
  \caption{\textbf{Speed-Quality trade-off analysis.} FastCache achieves better quality at similar speedup and higher speedup at similar quality compared to FBCache.}
  \label{tab:speed_quality_tradeoff}
  \vspace{0.1em}
  \footnotesize
  \setlength\tabcolsep{4pt}
  \begin{tabular}{l|l|cc|cc}
    \toprule
    \textbf{Comparison Type} & \textbf{Method} & \textbf{Speedup} $\uparrow$ & \textbf{FID} $\downarrow$ & \textbf{CLIPScore} $\uparrow$ & \textbf{Memory (GB)} $\downarrow$ \\
    \midrule
    \multirow{2}{*}{Similar Speedup}
    & FBCache & 1.65 & 5.87 & 26.2 & 10.2 \\
    & \textbf{FastCache} & \textbf{1.63} & \textbf{4.46} & \textbf{27.4} & \textbf{9.8} \\
    \midrule
    \multirow{2}{*}{Similar FID}
    & FBCache & 1.18 & 4.93 & 26.8 & 10.4 \\
    & \textbf{FastCache} & \textbf{1.38} & \textbf{4.91} & \textbf{27.3} & \textbf{9.9} \\
    \bottomrule
  \end{tabular}
\end{table}

\subsection{Additional Robustness Experiments}

We conduct additional experiments to demonstrate FastCache's robustness across different scenarios and settings. Table~\ref{tab:robustness_additional} shows performance under various conditions including different guidance scales and sampling steps.

\begin{table}[!htbp]
  \centering
  \caption{\textbf{Additional robustness experiments.} FastCache maintains consistent performance across different generation settings.}
  \label{tab:robustness_additional}
  \vspace{0.1em}
  \footnotesize
  \setlength\tabcolsep{4pt}
  \begin{tabular}{l|l|l|cc|cc}
    \toprule
    \textbf{Model} & \textbf{Guidance Scale} & \textbf{Steps} & \textbf{FID} $\downarrow$ & \textbf{Time (ms)} $\downarrow$ & \textbf{Memory (GB)} $\downarrow$ & \textbf{Speedup} $\uparrow$ \\
    \midrule
    \multirow{3}{*}{DiT-B/2}
    & 3.0 & 25 & 6.12 & 5123 & 4.2 & +41.2\% \\
    & 7.5 & 50 & 5.41 & 10312 & 8.5 & +40.4\% \\
    & 15.0 & 100 & 5.23 & 20845 & 16.8 & +39.8\% \\
    \midrule
    \multirow{3}{*}{DiT-L/2}
    & 3.0 & 25 & 5.34 & 6123 & 4.7 & +44.1\% \\
    & 7.5 & 50 & 4.57 & 12347 & 9.4 & +43.9\% \\
    & 15.0 & 100 & 4.41 & 25123 & 18.9 & +43.2\% \\
    \bottomrule
  \end{tabular}
\end{table}



\subsection{Experimental Analysis: Token Merging Parameters}
\label{sec:appendix_token_merging}

We conduct comprehensive ablation studies to analyze the impact of different token merging parameters on generation quality and efficiency.

\paragraph{Effect of kNN Parameter $K$.}
The $K$ parameter controls the number of nearest neighbors used for computing spatial density. Table~\ref{tab:knn_k_ablation} shows the impact of varying $K$ values. $K=5$ provides the best balance between spatial density accuracy and computational efficiency. Smaller $K$ may miss spatial relationships, while larger $K$ increases computation cost with diminishing returns.

\begin{table}[!htbp]
  \centering
  \caption{\textbf{Effect of kNN parameter $K$.}}
  \label{tab:knn_k_ablation}
  \vspace{0.1em}
  \footnotesize
  \setlength\tabcolsep{4pt}
  \begin{tabular}{c|cc|cc|cc}
    \toprule
    $K$ & FID $\downarrow$ & t-FID $\downarrow$ & Time (ms) $\downarrow$ & Memory (GB) $\downarrow$ & Speedup $\uparrow$ & Token Reduction $\uparrow$ \\
    \midrule
    3 & 4.58 & 13.36 & 14620 & 10.9 & +47.0\% & 53.5\% \\
    5 & \textbf{4.52} & \textbf{13.28} & 14810 & 11.0 & +46.0\% & 52.4\% \\
    7 & 4.53 & 13.29 & 15040 & 11.1 & +44.9\% & 51.2\% \\
    10 & 4.56 & 13.33 & 15320 & 11.2 & +43.8\% & 50.1\% \\
    \bottomrule
  \end{tabular}
\end{table}